\documentclass{IEEEtran}
\usepackage{graphicx}
\usepackage[utf8]{inputenc}
\usepackage{amsmath,amssymb}
\usepackage{float}
\usepackage{cite}
\usepackage{epsfig}
\usepackage{graphics}
\usepackage{graphicx}
\usepackage{epstopdf}
\usepackage{subcaption}
\usepackage{algorithm} 
\usepackage{algpseudocode} 
\usepackage{pifont}
\usepackage[utf8]{inputenc}
\usepackage{amsmath,amssymb}
\usepackage{epsfig}
\usepackage{graphics}
\usepackage{graphicx}
\usepackage{color}
\usepackage{latexsym}
\usepackage{epstopdf}
\usepackage{subcaption}
\usepackage{algorithm}
\usepackage{algpseudocode}
\usepackage{pifont}
\usepackage[utf8]{inputenc}
\usepackage[english]{babel}
\usepackage{bbm}
\usepackage{amsthm}
\newtheorem{theorem}{Theorem}[]

\newtheorem{lemma}[]{Lemma}
\newtheorem{proposition}{Proposition}[]

\newtheorem{assumption}{Assumption}
\newtheorem{remark}{Remark}
 \newcommand\subparagraph{}
\usepackage{url}

\usepackage{textcomp}
\def\BibTeX{{\rm B\kern-.05em{\sc i\kern-.025em b}\kern-.08em
    T\kern-.1667em\lower.7ex\hbox{E}\kern-.125emX}}

\allowdisplaybreaks
\begin{document}
%
\title {Online Reinforcement Learning of Optimal Threshold Policies for Markov Decision Processes  \thanks{This paper is a substantially 
expanded and revised version of the work in \cite{roy2019structure}.}}
%
%
%
\author{Arghyadip~Roy,~\IEEEmembership{Member,~IEEE,}
        Vivek~Borkar,~\IEEEmembership{Fellow,~IEEE,}
        Abhay~Karandikar,~\IEEEmembership{Member,~IEEE}
        ~and Prasanna~Chaporkar,~\IEEEmembership{Member,~IEEE}
\thanks{Arghyadip Roy was with Department
of Electrical Engineering, Indian Institute of Technology Bombay,
Mumbai, 400076, India when the work was done and is currently with Coordinated Science Laboratory, University of Illinois at Urbana-Champaign, 61820, USA. 
e-mail: arghyad4@illionois.edu.
Vivek Borkar, Abhay Karandikar and Prasanna Chaporkar are with the Department
of Electrical Engineering, Indian Institute of Technology Bombay. e-mail: {$\lbrace$borkar,karandi,chaporkar$\rbrace$}@ee.iitb.ac.in.
Abhay Karandikar is currently Director, Indian Institute of Technology Kanpur (on leave from IIT Bombay), Kanpur, 208016, India. e-mail:karandi@iitk.ac.in.}}
\maketitle

\begin{abstract}
To overcome the \textit{curses of dimensionality and modeling}
of Dynamic Programming (DP) methods to solve Markov Decision Process (MDP) problems,
Reinforcement Learning (RL) methods are adopted in practice.  
Contrary to traditional RL algorithms
which do not consider the structural properties of the optimal policy, we propose a structure-aware learning algorithm to 
exploit the ordered multi-threshold structure of the optimal policy, if any. We prove the asymptotic convergence of the proposed algorithm to the optimal policy. Due to the reduction 
in the policy space, the proposed algorithm provides remarkable improvements in storage and computational complexities over classical
RL algorithms. 
Simulation results establish that the proposed algorithm converges faster than other RL algorithms. 
\end{abstract}


\begin{IEEEkeywords}
Markov Decision Process, Stochastic Approximation Algorithms, Reinforcement
Learning, Stochastic Control, Online Learning of Threshold Policies.
\end{IEEEkeywords}

%
\IEEEpeerreviewmaketitle

\section{Introduction}\label{sec:intro}
Markov Decision Process (MDP)\cite{puterman2014markov} is a framework which is widely used for the optimization of stochastic systems (e.g., queue \cite{ccil2009effects}, inventory and production management) 
to make optimal temporal decisions. 
Dynamic Programming (DP) methods \cite{puterman2014markov} to compute the \textit{optimal policy} suffer from the \textit{curse of dimensionality}\cite[Chapter~4.1]{powell2007approximate},\cite{bellman1957dynamic} in many practical applications 
as they are computationally 
inconvenient due to  extremely high dimension of the iterates.
Furthermore, they suffer from the \textit{curse of modeling} since the knowledge of the underlying transition probabilities
(which often depend on the statistics of unknown system parameters) required by DP methods, may not be available beforehand. 
RL techniques \cite{sutton1998reinforcement} address the 
curse of modeling by learning the optimal policy iteratively. 
RL algorithms do not require any prior knowledge regarding the transition probabilities of the underlying model. 
RL being sampling based, updates only one component at a time, reducing per iterate computation at the expense of speed. Examples are:
Q-learning \cite{watkins1992q} 
iteratively evaluates the Q-function of every state-action 
pair using a combination of \textit{exploration} and \textit{exploitation}. 
In \cite{jin2018q,wei2019model,jafarnia2020model}, upper confidence bound based exploration improves the convergence speed over classical $\epsilon$-greedy exploration.
PDS learning algorithm \cite{powell2007approximate,salodkar2008line}
removes the requirement of action exploration
and results in faster convergence than 
Q-learning. 
Virtual Experience (VE) learning algorithm in \cite{mastronarde2012joint} updates multiple PDSs at a time. Faster convergence is achieved at the cost of increased computational complexity.

However, popular RL techniques \cite{watkins1992q,sutton1998reinforcement,borkar2005actor,powell2007approximate,salodkar2008line,bertsekas1995dynamic}
do not exploit the known structural properties of the optimal policy if any, and
consider the set of all policies as the policy search space. However, in operations research and communications literature,
various structural properties of the optimal policy including threshold structure, transience of certain states 
and index rules 
\cite{agarwal2008structural,smith2002structural} are often established
using monotonicity, convexity/concavity and sub-modularity/super-modularity properties of value functions of states.
If one can exploit these structural 
properties to reduce the search space while learning,
then faster convergence can be achieved with a reduction in the computational complexity. 
A few works in the literature \cite{kunnumkal2008exploiting,fu2012structure,ngo2010monotonicity,sharma2018accelerated} focus on the exploitation of the structural properties \cite{smith2002structural} while learning the optimal policy. 
Q-learning based approaches in \cite{kunnumkal2008exploiting,ngo2010monotonicity}, in every iteration, 
project the value functions to guarantee the monotonicity in system state. 
Although improved convergence speed is obtained, the per-iteration computational complexity does not improve over Q-learning.  In \cite{fu2012structure}, a 
scheme based on piecewise linear approximation of the value function, is proposed.
However, as the approximation becomes better, the complexity increases. In \cite{chakravorty2019remote,subramanian2019renewal}, a Stochastic Approximation (SA)\cite{borkar2008stochastic} approach based on simultaneous perturbation is proposed to compute the optimal thresholds. It uses a combination of ideas from renewal theory and Monte Carlo simulation. 
The low complexity Q-learning algorithm  proposed in \cite{liu2020sampled} exploits the threshold nature of the optimal policy. The proposed algorithm estimates the optimal policy for a subset of sets, and then policy interpolation is performed for the unvisited states. The performance of policy interpolation is further improved by a policy refinement. However, the proposed algorithm provides a near-optimal  policy  after  a  finite  number  of visits  to  a  set  of  state-action pairs.
\par


We consider a scenario where the optimal policy has a multi-threshold structure, and the thresholds for different events are ordered. Therefore, learning the optimal policy is equivalent to learning the value of threshold for each event.
Motivated by this, we propose a Structure-Aware Learning for MUltiple Thresholds 
(SALMUT) algorithm which considers only the set of ordered threshold policies.
We consider a two timescale approach 
where the value functions of the states and threshold parameters are updated in faster and slower timescale, respectively. Threshold parameters are updated 
based on the gradients of the average reward w.r.t the threshold.
The proposed scheme results
in reductions in storage and computational complexities (by amortizing the computation, and hence the complexity of the original problem, over several iterations) compared to traditional RL schemes.
We prove that the proposed algorithm converges to the optimal policy
asymptotically.
Simulation results 
exhibit that SALMUT converges faster than classical RL schemes due to reduction in policy search space. 
The techniques in this paper can be employed
to learn the optimal policy in problems
involving optimality of threshold 
policies \cite{agarwal2008structural, sinha2012optimal,koole1998structural,brouns2006optimal,
 ngo2009optimality}.
We illustrate this using an example of
a finite buffer multi-server queue with multiple customer classes (similar to \cite{ccil2009effects,ccil2007structural}).
\par
The proposed SALMUT algorithm can be adopted  for single threshold case in \cite{roy2019structure}, without any 
modification. In \cite{roy2019low}, a structure-aware learning algorithm learns a single parameterized threshold where the thresholds for different parameter values are independent of each other. However, in this paper, the thresholds have to satisfy certain ordering constraints and hence, 
are not independent. Therefore, the threshold update scheme in the slower timescale and corresponding convergence behavior differ significantly from \cite{roy2019low}.
To the best of our knowledge, contrary to other works 
\cite{kunnumkal2008exploiting,fu2012structure,ngo2010monotonicity,sharma2018accelerated}, we for the first time consider the threshold vector as a parameter while learning, to reduce a non-linear iteration (involving maximization over a set of actions) into a linear iteration for a quasi-static value of threshold and thereby, achieve a significant reduction in the per-iteration computational complexity.
\section{System Model \& Problem Formulation}\label{sec:sysmod}
Consider a continuous time MDP problem where we aim to obtain the optimal control policy in a multi-event scenario.
Let the system state in the state space ($\mathcal{S}\times \mathcal{I}$, say) be $(s,i)$
where $s\in\{0,1,\ldots,W\}$, 
and $i\in \{0,1,\ldots,N\}$ denotes the event type.
Since we have a finite-state regular Markov chain \cite{wolff1989stochastic},
it is sufficient to observe the system state only at the decision epochs \cite{kumar2012discrete}. 
Let the action space $\mathcal{A}$
consist of two actions, viz., $A_1$, and 
$A_2$.
The transition probability from state $(s,i)$ to $(s',i')$ under action $a$ ($p((s,i),(s',i'),a)$, say) can be factored into two parts, viz., the deterministic and the probabilistic transitions due to the chosen action ($p(s,s',a)$, say) and 
the next event, respectively.
Let the mean transition rate from state $(s,i)$ be denoted by $v(s)$.
Let the arrival times of events in state $s$ be independent exponentially distributed with means $\lambda_i(s)$ respectively, so that the next event is $i$ with probability $\frac{\lambda_i(s)}{\sum_{j=0}^N\lambda_j(s)}$.\\
Now, 
    $v(s)=\sum_{i=0}^N \lambda_i(s)$.
Hence,
\begin{equation*}
    p((s,i),(s',i'),a)=p(s,s',a)\frac{\lambda_{i'}(s')}{v(s')}.
\end{equation*}
Let the non-negative reward rate ($r((s,i),a)$, say) obtained by choosing action $A_2$ for event $i$ be $R_i$, where 
$R_i>R_j$ for $i<j$. 
Therefore,
$r((s,i),a)=
  R_i\mathbbm{1}_{\{a=A_2\}}$. 
  Let the non-negative cost rate in state $(s,i)$ be $h(s)$ (independent of $i$). 
\par 
Let $\mathcal{Q}$ be the set of stationary policies, 
Since the zero state is reachable from any state with positive probability, the underlying Markov chain is unichain and hence, a unique stationary distribution exists. Let the infinite horizon average reward (independent of the initial state) under policy $Q\in \mathcal{Q}$ be ${\rho}_Q$.
We aim to maximize
   $\rho_Q=\lim_{t\to \infty}\frac{1}{t}\mathbb{E}_Q [R(t)]$,
where $R(t)$ is the total reward till time $t$.
The DP equation describing the necessary condition for optimality in a semi-Markov decision process ($\forall \{(s,i),(s',i')\} \in \mathcal{S}\times \mathcal{I}$) 
is 
\begin{equation*}
\begin{split}
    \bar{V}(s,i)=&\max_{a \in \mathcal{A}}[r((s,i),a)+\sum_{s',i'}p((s,i),(s',i'),a)\bar{V}(s',i')\\&-\rho\bar{\beta}((s,i),a)]-h(s),
    \end{split}
\end{equation*}
where $\bar{V}(s,i)$, $\rho$ and $\bar{\beta}((s,i),a)$ denote 
the value function of state $(s,i)$, the optimal average reward and the mean transition time from state $(s,i)$ upon choosing action $a$, respectively. We rewrite the DP equation after substituting the values of $r((s,i),a)$ and transition probabilities as 
\begin{equation*}
\begin{split}
\bar{V}(s,i)=&\max_{a \in \mathcal{A}}\big[R_i\mathbbm{1}_{\{a=A_2\}}+\sum_{s',i'}p(s,s',a)\frac{\lambda_{i'}(s')}{v(s')}\bar{V}(s',i')\\&-\rho\bar{\beta}((s,i),a)\big]-h(s).    
\end{split}
\end{equation*}
We define $V(s):=\sum_{i=0}^N \frac{\lambda_i(s)}{v(s)}\bar{V}(s,i).$
Therefore, the following relations hold.
\begin{equation*}
\begin{split}
\bar{V}(s,i)=&\max_{a \in \mathcal{A}}\big[R_i\mathbbm{1}_{\{a=A_2\}}+\sum_{s'}p(s,s',a)V(s')\\& -\rho\bar{\beta}((s,i),a)\big]-h(s),
\end{split}
\end{equation*}
and (since $\frac{\lambda_i(s)}{v(s)}$ is independent of $a$)
\begin{equation}\label{eq:reduction}
\begin{split}
V(s)=&\max_{a \in \mathcal{A}}\big[\sum_i \frac{\lambda_{i}(s)}{v(s)} R_i\mathbbm{1}_{\{a=A_2\}}+\sum_{s'}p(s,s',a)V(s')\\&-\rho\bar{\beta}(s,a)\big]-h(s),
\end{split}
\end{equation}
where $\bar{\beta}(s,a)=\sum\limits_{i}\frac{\lambda_i(s)}{v(s)}\bar{\beta}((s,i),a)$.
Using (\ref{eq:reduction}), instead of considering  $(s,i)\in \mathcal{S}\times \mathcal{I}$, we can consider the system state as $s \in \mathcal{S}$ with value function $V(s)$ and transition probability $p(s,s',a)$, and the analysis remains unaffected. However, in this model, the reward rate is the weighted average of original reward rates. 
The sojourn times being exponentially distributed, 
following \cite{puterman2014markov}, we obtain 
\begin{equation}\label{eq:wlog}
\begin{split}
    V(s)=&\max_{a \in \mathcal{A}}[\sum_i \frac{\lambda_{i}(s)}{v(s)} R_i\mathbbm{1}_{\{a=A_2\}}+\sum_{s'}p(s,s',a)V(s')]\\&-h(s)-\rho,
    \end{split}
\end{equation}
where $p(s,s,a)=1-\sum_{s'\neq s} p(s,s',a)$.
 (\ref{eq:wlog}) is the DP equation for an equivalent discrete time MDP having controlled transition probabilities $p(s,s',a)$ which is used throughout the rest of the paper. 
This problem can be solved using Relative Value Iteration Algorithm (RVIA) as follows.
\begin{equation}\label{eq:RVIA}
\begin{split}
   V_{n+1}(s)=&\max_{a\in \mathcal{A}}[\sum_i \frac{\lambda_{i}(s)}{v(s)} R_i\mathbbm{1}_{\{a=A_2\}}+\sum_{s'}p(s,s',a)V_n(s')]\\&-V_n(s^*)-h(s),
\end{split}
\end{equation}
where $s^*\in \mathcal{S}$ is a fixed state and $V_n(s)$ is the estimate of value function of state $s$ at $n^{\rm {th}}$ iteration.

\section{Exploitation of structural properties in RL}\label{sec:salmut}
We assume that the optimal policy is of threshold-type where it is optimal to choose action $A_2$ only upto a threshold $\tau(i) \in \mathcal{S}$ which is a non-increasing function of $i$.
In this section, we propose an RL algorithm which exploits the knowledge regarding the existence of a threshold-based optimal policy.
\par Given that the optimal policy is threshold in nature where the optimal action changes from $A_2$ to $A_1$ at $\tau(i)$ for $i^{\rm{th}}$ event, the knowledge of $\tau(0), \ldots, \tau(N)$ uniquely characterizes the optimal policy. However, computation of these threshold parameters requires the knowledge of the event probabilities in state $s$ (governed by $\lambda_i(s)$). When the $\lambda_i(.)$s are unknown, then we can learn these ordered thresholds instead of learning the optimal policy from the set of all policies including the non-threshold policies. We devise an iterative update rule for a threshold vector of dimensionality $(N+1)$ 
so that the threshold vector iterate converges to the optimal threshold vector. \par
We consider the set of threshold policies where the thresholds for different events are ordered ($\tau(i)\ge \tau(j)$ for $i<j$) and represent them as policies parameterized by the threshold vector $\boldsymbol {\tau}=[\tau(0),\tau(1),\ldots, \tau(N)]^T$ where $\tau(0)\ge \tau(1)\ge\ldots \ge \tau(N)$. In this context, we redefine the notations associated with the MDP to reflect their dependence on $\boldsymbol{\tau}$. We intend to compute the gradient of the average reward  w.r.t $\boldsymbol{\tau}$ and improve the policy by updating $\boldsymbol{\tau}$ in the direction of the gradient.
Let the transition probability from state $s$ to state $s'$ corresponding to the threshold vector $\boldsymbol{\tau}$ be $P_{ss'}(\boldsymbol{\tau})$.
Hence, 
 $P_{ss'}(\boldsymbol{\tau})=P(X_{n+1}=s'|X_n=s,\boldsymbol{\tau}).$
Let the value function of state $s$, the average reward of the Markov chain and the stationary probability of state $s$ parameterized by  $\boldsymbol{\tau}$ be denoted by $V(s,\boldsymbol{\tau})$, $\sigma(\boldsymbol{\tau})$ and $\pi(s, \boldsymbol{\tau})$, respectively.

The optimal policy can be computed using  (\ref{eq:RVIA}) if we know the state transition probabilities and $\lambda_i(s)$s. When these parameters are unknown, theory of SA \cite{borkar2008stochastic} enables us to replace the expectation operation in  (\ref{eq:RVIA}) by averaging over time and still converge to the optimal policy. Let $a(n)$ be a positive step-size sequence satisfying 
$\sum_{n=1}^\infty a(n)=\infty;
\sum_{n=1}^\infty (a(n))^2<\infty.$
Let $b(n)$ be another step-size sequence which apart from the above properties 
satisfies
    $\lim \limits_{n\to \infty} \frac{b(n)}{a(n)}=0$.
We update the value function of the system state (based on the type of event) at any given iteration and keep the value functions of other states unchanged. Let $S_n$ be the state of the system at $n^{\rm {th}}$ iteration. 
Let
    $\gamma(s,n)=\sum _{x=1}^n \mathbbm{1}_{\{S_x=s\}}.$
For a fixed threshold vector $\boldsymbol{\tau}$ (i.e., a fixed policy), $\max$ operator in (\ref{eq:wlog}) goes away, and the resulting system becomes a linear system. Therefore we have
\begin{equation*}
\begin{split}
    V(s)=&\sum_i \frac{\lambda_{i}(s)}{v(s)} R_i\mathbbm{1}_{\{l(s,\tau(i))=A_2\}}+\sum_{s'}p(s,s',a)V(s')\\&-h(s)-\rho,
    \end{split}
\end{equation*}
where $l(s,\tau(i))=A_2$ if $s<\tau(i)$, $l(s,\tau(i))=A_1$ else. 

The update of value function of state $s$ (corresponding to the $i^{\rm {th}}$ event) is done using the following scheme:
\begin{equation}\label{eq:primal}
\begin{split}
 V_{n+1}(s,\boldsymbol{\tau})=&(1-a(\gamma(s,n)))V_n(s,\boldsymbol{\tau})+a(\gamma(s,n))[-h(s)+\\&R_i\mathbbm{1}_{\{l(s,\tau(i))=A_2\}}+V_n(s',\boldsymbol{\tau})-V_n(s^*,\boldsymbol{\tau})],\\
V_{n+1}(\tilde{s},\boldsymbol{\tau})&=V_{n}(\tilde{s},\boldsymbol{\tau}), \forall \tilde{s}\neq s,
\end{split}
\end{equation}
where $V_n(s,\boldsymbol{\tau})$ denotes the value function of state $s$ at $n^{\rm {th}}$ iteration provided the threshold vector is $\boldsymbol{\tau}$. This is known as the primal RVIA which is performed in the faster timescale. 
The scheme (\ref{eq:primal}) works for a fixed value of threshold vector. To obtain the optimal value of $\boldsymbol{\tau}$, the threshold vector needs to be iterated in a slower timescale $b(.)$.
The idea is to learn the optimal threshold vector by computing $\nabla \sigma(\boldsymbol{\tau})$ 
and update the value of threshold in the direction of the gradient. This scheme is similar to stochastic gradient routine:
\begin{equation}\label{eq:nabla_2}
    \boldsymbol{\tau}_{n+1}=\boldsymbol{\tau}_n+b(n)\nabla \sigma(\boldsymbol{\tau}_n),
\end{equation}
where $\boldsymbol{\tau}_n$ is the threshold vector at $n^{\rm {th}}$ iteration. 
Conditions on step sizes ensure that the value function and threshold vector iterates are updated in different timescales. From the slower timescale, the value functions seem to be quasi-equilibrated, whereas form the faster timescale, the threshold vector appears to be quasi-static (known as the ``leader-follower" behavior). \par 
Given a threshold vector $\boldsymbol{\tau}$, it is assumed that the transition from state $s$ for  $i^{\rm {th}}$ event is driven by the rule $P_1(s'|s)$ if $s<\tau(i)$ and by the rule $P_0(s'|s)$ otherwise. Under rule $P_1(s'|s)$, the system moves from state $s$ to state $s'=s+1$ following action $A_2$ if $s<\tau(i)$. On the other hand, rule $P_0(s'|s)$ dictates that the system remains in state $s$ following action $A_1$. For a fixed $\boldsymbol{\tau}$,  (\ref{eq:primal}) is updated using the above rule. Let $g_s({\tau}(i))=R_i\mathbbm{1}_{\{l(s,\tau(i))=A_2\}}-h(s)$.
 The following assumption is made on $P_{ss'}(\boldsymbol{\tau})$ and $g_s(\boldsymbol{\tau})$ to embed the discrete parameter $\boldsymbol{\tau}$ into a continuous domain later.
\begin{assumption}\label{ass:1}
$P_{ss'}(\boldsymbol{\tau})$ and $g_s(\boldsymbol{\tau})$ are bounded and twice differentiable functions of $\boldsymbol{\tau}$. It has bounded first and second derivatives.
\end{assumption}
For a given event, the threshold policy chooses the rule $P_1(.|.)$ upto a threshold 
and follows the rule $P_0(.|.)$, thereafter. Therefore the threshold policy is defined at discrete points and does not satisfy Assumption \ref{ass:1} as the derivative is undefined. To address this issue, we propose an approximation ($\approx$ interpolation to  continuous domain) of the threshold policy, which resembles a step function, so that the derivative exists at every point. This results in a randomized policy which in state $s$, chooses policies $P_0(s'|s)$ and $P_1(s'|s)$ with probabilities $f(s,\boldsymbol{\tau})$ and $1-f(s,\boldsymbol{\tau})$, respectively. In other words,
\begin{equation}\label{eq:appx}
P_{ss'}(\boldsymbol{\tau}) \approx P_0(s'|s)f(s,\boldsymbol{\tau})+P_1(s'|s)(1-f(s,\boldsymbol{\tau})).    
\end{equation}
Intuitively,  $f(s,\boldsymbol{\tau})$
should allocate similar probabilities to $P_0(s'|s)$ and $P_1(s'|s)$ near the threshold. As we move away towards the left (right) direction, the probability of choosing $P_0(s'|s)$
($P_1(s'|s)$) should decrease. 
The following function is chosen as a convenient approximation as it is continuously differentiable and the derivative is non-zero at every point.
\begin{equation}\label{eq:sigmoid}
    f(s,{\tau}(i))=\frac{e^{(s-{\tau}(i)-0.5)}}{1+e^{(s-{\tau}(i)-0.5)}}.
\end{equation}
Similar to (\ref{eq:appx}), we approximate $g_s(\boldsymbol{\tau})$ as
\begin{equation*}
g_s({\tau}(i))\approx -f(s,{\tau}(i))h(s)+(1-f(s,{\tau}(i)))(R_i-h(s)).   
\end{equation*}
   

\begin{remark}
Note that although the state space is discrete, individual threshold vector component iterates may take values in the continuous domain. However, only
an ordinal comparison dictates which action needs to be chosen in the current state.
\end{remark}
\begin{remark}
Instead of the sigmoid function in (\ref{eq:sigmoid}), the function $f(s,{\tau}(i))=1.\mathbbm{1}_{\{s\ge {\tau}(i)+1\}}
+(s-{\tau}(i)) \mathbbm{1}_{\{s< {\tau}(i)<s+1\}}$  which uses approximation only when $s< {\tau}(i)<s+1$, could have been chosen.
This function suffers less approximation error than that of  (\ref{eq:sigmoid}). However, it may lead to slow convergence since the derivative of the function and hence the gradient becomes zero outside 
$s< {\tau}(i)<s+1$.
Although the derivative of sigmoid function decays exponentially fast too, we observe in simulations that the convergence behavior with (\ref{eq:sigmoid}) is better.
\end{remark}
Under Assumption \ref{ass:1}, the following proposition \cite[Proposition~1]{marbach2001simulation} provides a closed form expression for the gradient of $\sigma(\boldsymbol{\tau})$. 
The proposition stated next is extended to policy gradient theorem in \cite{sutton2000policy}.

\begin{proposition}
\begin{equation}\label{eq:nabla_1}
    \nabla \sigma(\boldsymbol{\tau})=\sum \limits_{s\in \mathcal{S}} \pi(s,\boldsymbol{\tau})(\nabla g_s(\boldsymbol{\tau})+\sum \limits_{s'\in \mathcal{S}}\nabla P_{ss'}(\boldsymbol{\tau})V(s',\boldsymbol{\tau})).
\end{equation}
\end{proposition}
\begin{remark}
The first term on the right of (\ref{eq:nabla_1}) was dropped earlier, the
reason given being that ‘the reward is independent of the parameter’.
This is incorrect because the indicator function multiplying the
reward is not so.
Nevertheless, it turns out 
(see Fig. \ref{fig:cost_1_step} and \ref{fig:cost_2_step})
that the performance is hardly affected by the omission of this term. This is presumably because this term is significant only in a small neighborhood of the threshold. 
This mistake is present in our earlier works \cite{roy2019structure,roy2019low} and similar remarks apply there.
\footnote{We thank Prof. Aditya Mahajan, McGill University for pointing out this error.}
\end{remark}
Based on the proposed approximation (\ref{eq:sigmoid}), we set to devise an online update rule for the threshold vector in the slower timescale $b(.)$, following  (\ref{eq:nabla_2}). We evaluate $\nabla P_{ss'}(\boldsymbol{\tau})$ as a representative of $\nabla \sigma(\boldsymbol{\tau})$ since the stationary probabilities in  (\ref{eq:nabla_1}) can be replaced by averaging over time. Using  (\ref{eq:appx}), we get
\begin{equation}\label{eq:appxnabla}
    \nabla P_{ss'}(\boldsymbol{\tau})=(P_0(s'|s)-P_1(s'|s)) \nabla f(s,\boldsymbol{\tau}).
\end{equation}
We incorporate a multiplying factor of $\frac{1}{2}$ in the right hand side of (\ref{eq:appxnabla}) since multiplication by a constant term does not alter the scheme. The physical significance of this operation is that
at every iteration, transitions following rules 
$P_1(.|.)$ and $P_0(.|.)$ are adopted with equal probabilities.$\nabla f(s,\boldsymbol{\tau})$ depends on the system state and threshold vector at any given iteration.
Similar online rule can be devised for $\nabla g_s (\boldsymbol{\tau})$ using an identical procedure. \par 
Based on this, when  $i^{\rm{th}}$ event occur, the online update rule 
for the $i^{\rm{th}}$ component of the threshold vector is 
as follows.
\begin{equation*}
\begin{split}
{\tau}_{n+1}(i)=&\Omega_i[{\tau}_{n}(i)+b(n)\nabla f(s,{\tau}_n(i))({(-1)}^{\beta_n}\hat{h}_{\beta_n}(s,i)+\\&{(-1)}^{\alpha_n}V_n(\hat{s},{\tau}_n(i)))],  
\end{split} 
\end{equation*}
where $\alpha_n$ and $\beta_n$ are i.i.d. random variables, each of which can take values $0$ and $1$ with equal probabilities. If $\alpha_n=1$, then the transition is governed by the rule $P_1(.|.)$, else by $P_0(.|.)$. In other words, the next state is $\hat{s}$ with probability $\alpha_n P_1(\hat{s}|s)+(1-\alpha_n)P_0(\hat{s}|s)$. Similarly, $\hat{h}_{\beta}(s,i)=-\beta h(s)+(1-\beta)(R_i-h(s))$ where the obtained reward is $-h(s)$ and $R_i-h(s)$ with equal probabilities. 
The projection operator $\Omega_i$ ensures that 
${\tau}(i)$ iterates remain bounded in a specific interval, as specified later. Recall that, we have assumed that the threshold for $i^{\rm{th}}$ event is a non-increasing function of $i$. Therefore, the $i^{\rm{th}}$ component of the threshold vector iterates should always be less than or equal to the $(i-1)^{\rm{th}}$ component. 
The first component of $\boldsymbol{\tau}$ is considered to be a free variable which can choose any value in $[0,W]$. 
The projection operator $\Omega_i,\{i> 0\}$ ensures that
${\tau}_n(i)$ remains bounded in $[0,{\tau}_{n}(i-1)]$.
To be precise, 
\begin{equation*}
\begin{split}
&\Omega_0: x \mapsto 0\vee(x\wedge W) \in [0, W],\\&
\Omega_i: x \mapsto 0\vee(x\wedge {\tau}(i-1)) \in [0,{\tau}(i-1)]. \quad \forall i> 0.
\end{split}
\end{equation*}
The framework of SA enables us to obtain the effective drift in (\ref{eq:appxnabla}) by performing averaging. 
Therefore the online RL scheme where the value functions and the threshold vector are updated in the faster and the slower timescale, respectively, is as follows. We suppress the parametric dependence of $V$ on $\tau$. 
\begin{equation}\label{eq:primal_1}
\begin{split}
 V_{n+1}(s)=&(1-a(\gamma(s,n)))V_n(s)+a(\gamma(s,n))[-h(s)+\\&R_i\mathbbm{1}_{\{l(s,\tau(i))=A_2\}}+V_n(s')-V_n(s^*)],\\
V_{n+1}(\tilde{s})=&V_{n}(\tilde{s}), \forall \tilde{s}\neq s,
\end{split}
\end{equation}
and
\begin{equation}\label{eq:dual_1}
\begin{split}
{\tau}_{n+1}(i)=&\Omega_i[{\tau}_{n}(i)+b(n)\nabla f(s,{\tau}_n(i))({(-1)}^{\beta_n}\hat{h}_{\beta_n}(s,i)+\\&{(-1)}^{\alpha_n} V_n(\hat{s}))], \\ 
{\tau}_{n+1}(i')=&{\tau}_{n}(i'), \ \forall i'< i, \\ 
{\tau}_{n+1}(i')=&\Omega_i [{\tau}_{n}(i')], \ \forall i'> i,
\end{split}
\end{equation}
where for current state $s$, transition to next state $s'$ in (\ref{eq:primal_1}) refers to a single run of a simulated chain as commonly seen in RL, and $\hat{s}$ in (\ref{eq:dual_1}) is determined separately following the probability distribution $\alpha_n P_1(\hat{s}|s)+(1-\alpha_n) P_0(\hat{s}|s)$. The immediate reward is determined using $-\beta_n h(s)+(1-\beta_n)(R_i-h(s))$. 
The physical significance of (\ref{eq:dual_1}) is that when $i^{\rm{th}}$ type of event occurs, then the $i^{\rm{th}}$ component of $\boldsymbol{\tau}$ is updated. However, since the components are provably ordered, we need to update the $i'^{\rm{th}}$ components too where
 $i'> i$. We no longer need to update the components
 for $i'< i$ since the order is already preserved while updating the $i^{\rm{th}}$ component.
 Also, in (\ref{eq:primal_1}), the reward function is taken to be $R_i\mathbbm{1}_{\{l(s,{\tau}(i))=A_2\}}$ when $i^{\rm {th}}$ event occurs. The expectation operation in  (\ref{eq:RVIA}) is mimicked by the averaging over time implicit in an SA scheme. Note that contrary to \cite{roy2019low} where due to the independence among the threshold parameters, only one threshold is updated at a time, in this paper, multiple threshold parameters may need to get updated to capture the ordering constraints.
 \begin{remark}
Instead of the two-timescale approach adopted in this paper, a multi-timescale approach where each individual threshold is updated in a separate timescale, may be chosen. However, since the updates of thresholds are coupled only through the ordering constraints, they can be updated in the same timescale. Moreover, in practice, a multi-timescale scheme may not work well since the fastest (slowest) timescale may be too fast (slow), leading to higher fluctuations and/or very slow speed.
 \end{remark}

 \begin{theorem}\label{theo:2}
Update rules (\ref{eq:primal_1}) and (\ref{eq:dual_1}) converge to the optimal policy almost surely (a.s.).
\end{theorem}
\begin{proof}
Proof is given in Appendix \ref{app:b}.
\end{proof}
\begin{remark}
Unlike the value function iterates, the threshold vector iterates do not require local clocks of individual elements for convergence.
However, all components of the threshold vector are required to get updated comparably often. In other words, their frequencies of update should be bounded away from zero which generally holds true for stochastic gradient approaches \cite[Chapter~7]{borkar2008stochastic}.
\end{remark}

Based on the foregoing analysis, we describe the resulting SALMUT algorithm in Algorithm \ref{algo:1}. 
On a decision epoch, 
 action $a$ is chosen based on the current value of threshold vector. Based on the event, value function of current state $s$ is updated then using (\ref{eq:primal_1}) in the faster timescale. The threshold vector is also updated following (\ref{eq:dual_1}) in the slower timescale.  
Note that the value function is updated one component at a time. However, multiple components of the threshold vector may need to be updated in a single iteration.
The scheme resembles actor-critic method \cite{borkar2005actor} with policy gradient for the actor part (see (\ref{eq:dual_1})) and post-decision framework for the critic (see (\ref{eq:primal_1})).
\begin{algorithm}
\small
\caption{Two-timescale SALMUT algorithm}\label{algo:1}
\label{NCalgorithm}
\begin{algorithmic}[1]
\State Initialize  $n \leftarrow 1$,
 ${V}(s) \leftarrow 0, \forall {s}\in \mathcal{S}$ and  $\boldsymbol{\tau} \leftarrow \vec{0}$.
\While {TRUE}
\If {$i^{\rm{th}}$ event occurs}
\State Choose action $a$ based on current value of ${\tau}(i)$. 
\EndIf
\State Update value function of state $s$ using  (\ref{eq:primal_1}).
\State Update threshold $\boldsymbol{\tau}$ using (\ref{eq:dual_1}).
\State Update $s\leftarrow s'$ and $n\leftarrow n+1$.
\EndWhile
\end{algorithmic}
\end{algorithm}
\begin{remark}
Even if there does not exist an optimal threshold policy for a given MDP problem, the techniques  in this paper can be applied to learn the best threshold policy (locally at least) asymptotically. Threshold policies are easy to implement and often provide comparable performances to that of the optimal policy, with a significantly lower storage complexity.  
\end{remark}
\section{Complexity Analysis}\label{sec:complexity}
In this section, we compare the storage and computational complexities of SALMUT algorithm with those of existing learning schemes 
and summarize in Table \ref{table:1}.
\begin{table}[ht]
\caption{Complexities of RL algorithms.}\label{table:1}
\centering
\begin{tabular}{|l||l||l|}
\hline
\textbf{Algorithm} & \textbf{Computational} & \textbf{Storage}\\
\textbf{} & \textbf{complexity} & \textbf{complexity}\\ \hline
Q-learning \cite{sutton1998reinforcement,watkins1992q} & $O(|\mathcal{A}|)$ & $O(|\mathcal{S}|\times|\mathcal{A}|)$\\ \hline
Monotone Q-learning \cite{kunnumkal2008exploiting,ngo2010monotonicity} & $O(|\mathcal{A}|)$ & $O(|\mathcal{S}|\times|\mathcal{A}|)$\\ \hline
MH-Q-learning & $O(\log |\mathcal{A}|)$ & $O(|\mathcal{S}|\times|\mathcal{A}|)$\\ \hline
OQL \cite{wei2019model} &  $O(\log |\mathcal{A}|)$ & $O(|\mathcal{S}|\times|\mathcal{A}|)$ \\ \hline
EEQL \cite{jafarnia2020model} &  $O(\log |\mathcal{A}|)$ & $O(|\mathcal{S}|\times|\mathcal{A}|)$ \\ \hline

PDS learning \cite{salodkar2008line,powell2007approximate} &  $O(|\mathcal{A}|)$ & $O(|\mathcal{S}|)$ \\ \hline
MH-PDS learning &  $O(\log |\mathcal{A}|)$ & $O(|\mathcal{S}|)$ \\ \hline
VE learning \cite{mastronarde2012joint} &  $O(|\mathcal{V}|\times|\mathcal{A}|)$ & $O(|\mathcal{S}|)$ \\ \hline
Grid learning \cite{sharma2018accelerated} &  $O(|\mathcal{W}|\times|\mathcal{A}|)$ & $O(|\mathcal{S}|)$ \\ \hline
Adaptive appx. learning \cite{fu2012structure} &  $O(q(\delta)|\mathcal{A}|)$ & $O(|\mathcal{S}|)$ \\ \hline
SALMUT  & $O(1)$ & $O(|\mathcal{S}|)$\\ \hline
\end{tabular}
\end{table}
Q-learning and PDS learning need to store the value function of every state-action pair and every PDS, respectively. While updating the value function, both choose the best one after evaluating $|\mathcal{A}|$ functions. Since only one state-action pair is updated at a time, the computational complexity associated with remaining operations is constant. Thus, the storage 
complexities of Q-learning and PDS learning are $O(|\mathcal{S}|\times|\mathcal{A}|)$ and $O(|\mathcal{S}|)$, respectively. The per-iteration computational complexity of Q-leanring and PDS learning is
$O(|\mathcal{S}|)$. Since at each iteration, the monotone Q-learning algorithm \cite{kunnumkal2008exploiting, ngo2010monotonicity} projects the policy obtained using Q-learning within the set of monotone policies, the complexities are identical to those of Q-learning. 
VE learning \cite{mastronarde2012joint} updates multiple PDSs at a time. Therefore, the computational complexity contains an additional term $|\mathcal{V}|$ which signifies the cardinality of the VE tuple. Similarly, grid learning \cite{sharma2018accelerated} and adaptive approximation learning \cite{fu2012structure} are associated with additional factors $|\mathcal{W}|$ and $q(\delta)$, respectively. $|\mathcal{W}|$ and $q(\delta)$ depend on the depth of a quadtree used for value function approximation and the approximation error threshold
$\delta$, respectively. Note that computational complexities of Q-learning and PDS learning can be reduced to $O(\log|\mathcal{A}|)$ using a max-heap implementation (MH-Q-learning and MH-PDS learning in Table \ref{table:1}) where the complexities of obtaining the
best action and updating a value are  $O(1)$ and $O(\log |\mathcal{A}|)$, respectively. The computational complexity of a max-heap implementation of Optimistic Q-learning (OQL) \cite{wei2019model} and Exploration Enhanced Q-learning (EEQL) \cite{jafarnia2020model} is $O(\log |\mathcal{A}|)$. The storage complexity of OQL and EEQL is $O(|\mathcal{S}|\times|\mathcal{A}|)$. It is not clear whether the knowledge of structural properties can be encoded in Q-learning easily since the range of the threshold is large, viz., the entire state space.
The storage complexity of SALMUT algorithm is $O(|\mathcal{S}|)$ as we need to store the value functions of states. 
SALMUT may require to update all components of the threshold vector at a time
(See ({\ref{eq:dual_1}})).  Furthermore, the update of value function involves the computation of a single function based on
the current threshold (See (\ref{eq:primal_1})).
Therefore, the per-iteration computational complexity is $O(1)$. Thus SALMUT provides significant improvements in storage and per-iteration computational complexities compared to other schemes. Note that the computational complexity of SALMUT does not depend on $|\mathcal{A}|$ and depends only on the number of events.
\section{Illustrative Example}\label{sec:example}
We consider a queuing system with $m$ identical servers, $N$ customer classes and a finite buffer of size $B$. We investigate an optimal admission control problem.
It is assumed that the arrival of class-$i$ customers is a Poisson process with mean $\lambda_i$. Let the service time be exponentially distributed with mean $\frac{1}{\mu}$, irrespective of the customer class. Note that presence of $m$ identical servers does
not change the nature of the model. Similar analysis holds for the single server case also. 
The state of the system is $(s,i)$ where $s\in\{0,1,\ldots,m+B\}$ denotes the total number of customers in the system and $i\in \{0,1,\ldots,N\}$ denotes the class type. 
Arrivals and departures of customers are taken as decision epochs.
We take $i=0$ as the departure event and $i=1,\ldots, N$
as an arrival of a $1,\ldots,N^{\rm{th}}$ class of customer. 
For example, states $(2,0)$ and $(2,1)$ correspond to a departure and an arrival of class 1 customer while there are 2 users in the system, respectively.
$\mathcal{A}$ consists of three actions, viz., blocking of an arriving user ($A_1$), 
admission of an arriving user ($A_2$)  and continue/do nothing ($A_0$, say) for departures.
When $s=m+B$, then the only feasible action for an arrival (i.e., $i>0)$ is $A_1$.
We have, $p(s,s',A_2)=\mathbbm{1}_{\{s'=s+1\}}$, $p(s,s',A_1)=\mathbbm{1}_{\{s'=s\}}$ and $p(s,s',A_0)=\mathbbm{1}_{\{s'=(s-1)^+\}}$.
 where $x^+=\max\{x,0\}$
and
    $v(s)=\sum_{i=1}^N \lambda_i+\min\{s,m\}\mu.$
Let $\lambda_0(s)=\min\{s,m\}\mu$ and 
$\lambda_i(s)=\lambda_i, \forall i\neq 0$. Note that $\lambda_1(s),\ldots,\lambda_N(s)$ do not depend on $s$. 
Furthermore,
$r((s,i),a)=
  R_i\mathbbm{1}_{\{a=A_2,i>0\}}.$
$R_0=0$ corresponds to a departure event.
In state $(s,i)$, a non-negative cost rate of $h(s)$ is incurred.
$h(s)$ and $h(s+1)-h(s)$ are increasing in $s$ (convex increasing in the discrete domain).


One application of this model is a
make-to-stock production system that produces $m$ items with $N$ demand classes\cite{ha1997inventory} and buffer size $B$. 
Satisfaction of a demand (requesting a single unit of the product) from class-$i$ gives rise to reward rate $R_i$. The production time is exponentially distributed with mean $\frac{1}{\mu}$. The inventory holding cost rate is $h(s)$. \par
Now, we derive that there exists a threshold based optimal policy which admits a class-$i$ customer only upto a threshold $\tau(i)$ 
which is a non-increasing function of $i$. We prove these properties using the following lemma.
The proof follows from \cite[Theorem~3.1]{koole2007monotonicity}. Detailed proof is given in Appendix \ref{app:a}. 
\begin{lemma}\label{lemma:a}
$V(s+1)-V(s)$ is decreasing in $s$.
\end{lemma}

\begin{theorem}\label{theo:1}
The optimal policy is of threshold-type where it is optimal to admit class-$i$ customers only upto a threshold $\tau(i) \in \mathcal{S}$ which is a non-increasing function of $i$.
\end{theorem}
\begin{proof}
For class-$i$ customers, if $A_1$ is optimal in state $s$, then $R_i+V(s+1)\le V(s)$ (Using (\ref{eq:wlog})). From Lemma \ref{lemma:a},
$V(s+1)-V(s)$ is decreasing in $s$. This proves the existence of a threshold $\tau(i)$ for class-$i$ customers.
Since 
$R_i>R_j$ for $i<j$, $R_i+V(s+1)\le V(s)$ implies 
$R_j+V(s+1)\le V(s)$. 
Therefore,
$\tau(i)$ is  a non-increasing function of $i$. 
\end{proof}
The subsequent lemmas
establish
the unimodality of the average reward with respect to $\boldsymbol{\tau}$. Hence, 
Theorem \ref{theo:2}
holds. Proofs are presented in Appendix \ref{app:c}.
\begin{lemma}\label{lemma:noninc}
$v_n(s+1)-v_n(s)$ is decreasing in $n$.
\end{lemma}
\begin{lemma}\label{lemma:unimodal}
$\sigma(\boldsymbol{\tau})$ is unimodal in $\boldsymbol{\tau}$.
\end{lemma}
For this particular problem, the improvement in computational complexity offered by SALMUT may not be significant since $|\mathcal{A}|=3$. However, in general, for a large $|\mathcal{A}|$, the improvement in computational complexity may be remarkable.
\section{Simulation Results}\label{sec:simu}
In this section, we 
compare the convergence speed of SALMUT algorithm 
with traditional RL algorithms. We simulate the finite buffer multi-server system with two customer classes. We take $\lambda_1=\lambda_2=1$ s$^{-1}$, $m=B=5$,
$R_1=20$, $R_2=10$, 
$h(s)=0.1s^2$,   $a(n)=\frac{1}{(\lfloor{\frac{n}{100}\rfloor+2})^{0.6}}$ and 
$b(n)=\frac{10}{n}$. We exclude initial 10 burn-in period values of the iterates. \par 


\begin{figure*}[t!]
    \centering
    \begin{subfigure}[t]{0.4\textwidth}
       \includegraphics[width=\textwidth]{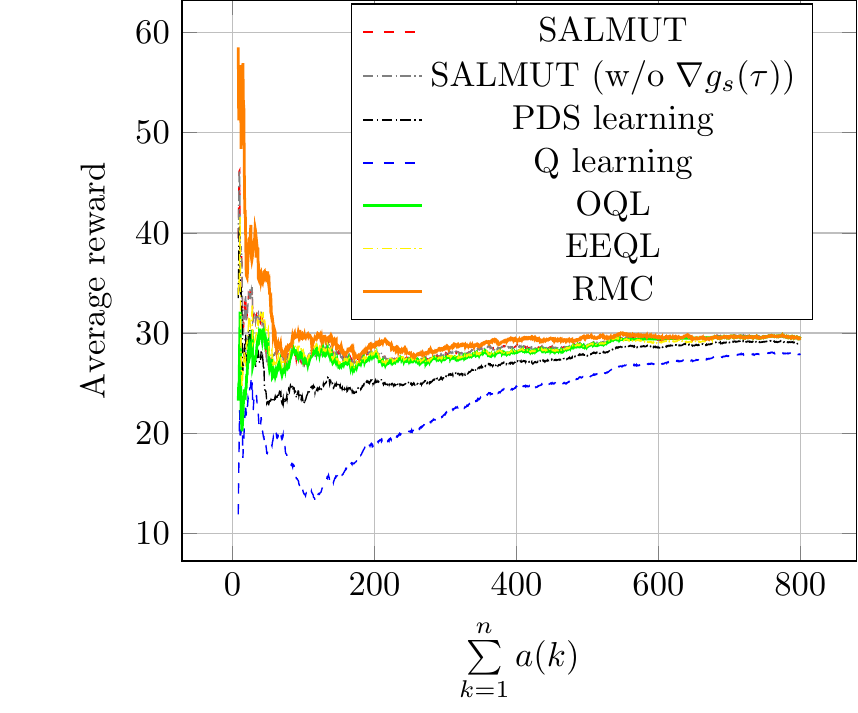}
        \caption{}
        \label{fig:cost_1_step}
        \end{subfigure}%
    ~
    \begin{subfigure}[t]{0.4\textwidth}
       \includegraphics[width=\textwidth]{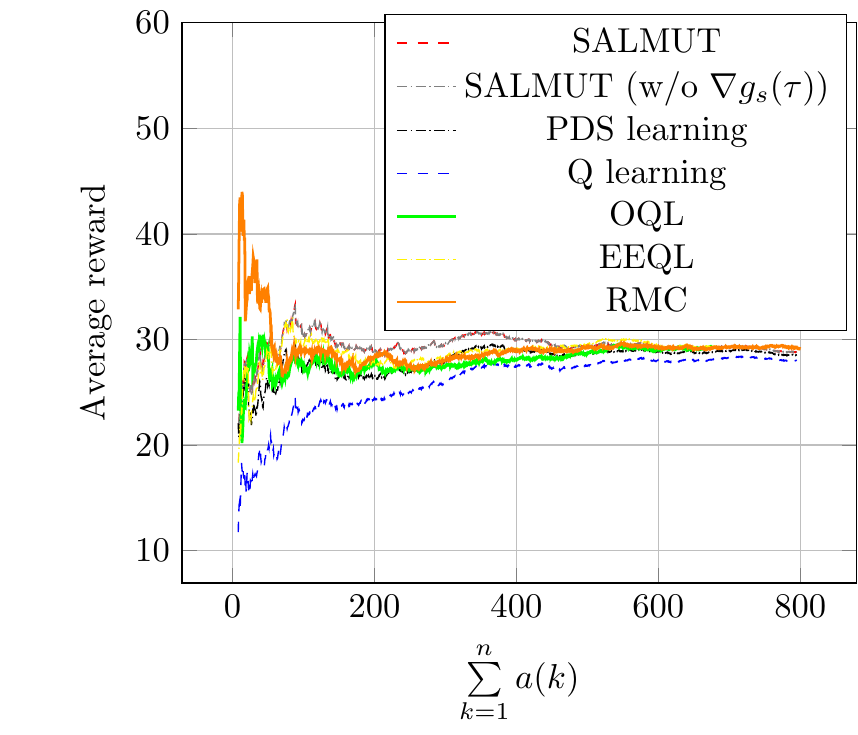}
        \caption{}
        \label{fig:cost_2_step}
        \end{subfigure}

\caption{Plot of average reward vs.  ($\sum_{k=1}^{n} a(k)$) for different algorithms ((a) $\mu=4 s^{-1}$, (b) $\mu=2 s^{-1}$).}
\end{figure*}
\par In practical cases, 
when the average reward of the system does not change much over a suitable window, we may conclude that stopping condition is met as
the obtained policy is close 
to the optimal policy with a high probability. The choice of window size $\sum _{k=p}^{p+n} a(k)$ over $n$ is to eliminate the effect of diminishing step size affecting the convergence behavior. 
We choose a window size of $50$
and stop when the ratio of maximum and minimum average rewards exceeds $0.95$. Fig. \ref{fig:cost_1_step} (\ref{fig:cost_2_step}) reveals that practical convergences for Q-learning, PDS learning and SALMUT algorithms
are achieved in 1180 (1180), 580 (1180) and 426 (580) iterations, respectively. 
Since unlike PDS learning, Q-learning is associated with exploration mechanism, 
PDS learning converges faster than Q-learning. However, SALMUT converges faster than both Q-learning and PDS learning algorithms since it operates on a smaller policy space (threshold policies only). 
Due to efficient exploration, the performances of  OQL \cite{wei2019model} and EEQL \cite{jafarnia2020model} are better than that of Q-learning and slightly worse than that of SALMUT. 
Since for a given sample path, the sequence of actions are same, the performance of MH-Q-learning (MH-PDS learning) is identical to that of Q-learning (PDS learning). The performance of adaptive approximation learning algorithm \cite{fu2012structure} is identical to that of PDS learning as batch update of PDSs is not possible. Therefore, we do not show their performances in the plot. The convergence behavior of Renewal Monte Carlo (RMC) \cite{chakravorty2019remote,subramanian2019renewal} algorithm (with discount factor=0.999) is slightly worse than that of SALMUT. Since the sampled Q-learning algorithm \cite{liu2020sampled} provides a near-optimal policy after a finite number of visits to a set of state-action pairs, the convergence behavior may not be comparable. 
\section{Conclusions \& Future Directions}\label{sec:conclude}
In this paper, we have proposed an RL algorithm which exploits the ordered multi-threshold nature of the optimal policy. 
The convergence of the proposed algorithm to the globally optimal threshold vector is established. 
The proposed scheme provides improvements in storage and computational complexities over traditional RL algorithms. Simulation results establish the improvement in convergence behavior with respect to state-of-the-art RL schemes.
In future, this can be extended to develop RL algorithms for constrained MDP problems
by updating the Lagrange Multiplier (LM) in a slower timescale \cite{borkar2008stochastic} than that of value functions. 
The LM and the threshold parameter can be updated in the same slower timescale without requiring a third timescale as they are independent of each other. Another possible future direction is to develop RL algorithms 
for restless bandits such as \cite{borkar2018reinforcement} since threshold policies often translate into index-based policies.  
\appendices
\section{Proof of Lemma \ref{lemma:a}}\label{app:a}
Proof techniques are similar to those of our earlier work \cite{roy2019structure}.
The optimality equation for value function is
\begin{equation*}
\begin{split}
&    \bar{V}(s,i)=\max_{a \in \mathcal{A}}[ R_i\mathbbm{1}_{\{a=A_2\}}+\sum_{s'}p(s,s',a)V(s')]-h(s).
\end{split}
\end{equation*}

In Value Iteration Algorithm (VIA), let the value function of state $s$ at $n^{\rm{th}}$ iteration be denoted by $v_n(s)$. We start with $v_0(s)=0, \ \forall s \in \mathcal{S}$. 
Therefore, $v_0(s+1)-v_0(s)=0$. 
Since (using definition of $V(s)$)
\begin{equation}\label{eq:app_1}
\begin{split}
&    v_{n+1}(s)=\sum_{i=1}^N \lambda_i \max\{R_i+v_n(s+1),v_n(s)\}\\&+\min\{s,m\}\mu v_n((s-1)^+)+(1-v(s))v_n(s)-h(s),
\end{split}
\end{equation}
and $h(s+1)-h(s)$ is increasing in $s$,
$v_1(s+1)-v_1(s)$ is decreasing in $s$.
Let us assume that $v_n(s+1)-v_n(s)$ is a decreasing function of $s$. We require to prove that $v_{n+1}(s+1)-v_{n+1}(s)$ is a decreasing function of $s$. Since $\lim \limits_{n\to \infty} v_n(s)=V(s)$, this implies the lemma.\par
We define $\hat{v}_{i,n+1}(s,a)$ and $\hat{v}_{i,n+1}(s),$ $\forall i \in \{1,2, \ldots, N\}$ as
 \begin{equation*}
\hat{v}_{i,n+1}(s,a)=
 \begin{cases}
      v_n(s) ,&a=A_1,\\
      R_i+v_n(s+1),&a=A_2.\\
\end{cases}
\end{equation*}
and $\hat{v}_{i,n+1}(s)=\max \limits_{a \in \mathcal{A}}\hat{v}_{i,n+1}(s,a)$.
Also, we define \\$Dv_{n}(s)=v_n(s+1)-v_n(s)$ and $D\hat{v}_{i,n}(s,a)=\hat{v}_{i,n}(s+1,a)-\hat{v}_{i,n}(s,a)$. Hence, we have,
 \begin{equation*}
D\hat{v}_{i,n+1}(s,a)=
 \begin{cases}
      Dv_n(s) ,&a=A_1,\\
      Dv_n(s+1),&a=A_2,\\
\end{cases}
\end{equation*}
and
 \begin{equation*}
D^2\hat{v}_{i,n+1}(s,a)=
 \begin{cases}
     D^2v_n(s) ,&a=A_1,\\
      D^2v_n(s+1),&a=A_2.\\
\end{cases}
\end{equation*}
Since $v_n(s+1)-v_n(s)$ is a decreasing function of $s$, $\hat{v}_{i,n+1}(s+1,a)-\hat{v}_{i,n+1}(s,a)$ is decreasing in $s$, $\forall i\in \{1,2,\ldots N\}$ and $\forall a \in \mathcal{A}$ . Let the maximizing actions for the admission of a class-$i$ customer in states $(s+2)$ and $s$ be denoted by $a_{i,1}\in \mathcal{A}$ and $a_{i,2}\in \mathcal{A}$, respectively. Therefore,
\begin{equation*}
\begin{split}
&2\hat{v}_{i,n+1}(s+1) \ge \hat{v}_{i,n+1}(s+1,a_{i,1})+\hat{v}_{i,n+1}(s+1,a_{i,2})\\&
=\hat{v}_{i,n+1}(s+2,a_{i,1})+\hat{v}_{i,n+1}(s,a_{i,2})+D\hat{v}_{i,n+1}(s,a_{i,2})\\&-D\hat{v}_{i,n+1}(s+1,a_{i,1}).
\end{split}
\end{equation*}
Let us denote $Z=D\hat{v}_{i,n+1}(s,a_{i,2})-D\hat{v}_{i,n+1}(s+1,a_{i,1})$. To prove that 
$\hat{v}_{i,n+1}(s+1)-\hat{v}_{i,n+1}(s)$ is non-increasing in $s$, we need to prove that 
$Z \ge 0$. We consider the following cases. 
\begin{itemize}
    \item $a_{i,1}=a_{i,2}=A_1$\\
    $Z=Dv_n(s)-Dv_n(s+1)=-D^2v_n(s)\ge 0$.
    \item $a_{i,1}=A_1, a_{i,2}=A_2$\\
    $Z=Dv_n(s+1)-Dv_n(s+1)= 0.$
        \item $a_{i,1}=a_{i,2}=A_2$\\
    $Z=Dv_n(s+1)-Dv_n(s+2)=-D^2v_n(s+1)\ge 0$.
    \item $a_{i,1}=A_2, a_{i,2}=A_1$\\
    $Z=Dv_n(s)-Dv_n(s+2)= -D^2v_n(s)-D^2v_n(s+1) \\
    \ge 0.$
\end{itemize}
To analyze the difference of second and third terms in Equation (\ref{eq:app_1}) corresponding to states $(s+1)$ and $s$, we consider two cases.
\begin{itemize}
    \item $s\ge m$:
    Both $m\mu(v_n(s)^+-v_n(s-1)^+)$ and $(1-\sum \limits_{i=1}^N \lambda_i-m \mu)(v_n(s+1)-v_n(s))$ are decreasing in $s$.
    \item $s<m$:
    The difference is equal to 
    \begin{equation*}
        s \mu (v_n(s)^+-v_n(s-1)^+)+(1-v(s+1))(v_n(s+1)-v_n(s)),
    \end{equation*}
    which is decreasing in $s$.
\end{itemize}
Since $h(s+1)-h(s)$ is increasing in $s$,
this proves that $v_n(s+1)-v_n(s)$ is decreasing in $s$. Hence,
$V(s+1)-V(s)$ is decreasing in $s$.

\section{Proof of Theorem \ref{theo:2}}\label{app:b}
We adopt the approach of viewing SA algorithms as a noisy discretization of a limiting Ordinary Differential Equation (ODE), similar to \cite{roy2019structure,roy2019low}. Step size parameters are viewed as discrete time steps. 
Standard assumptions on step sizes 
ensure that the errors due to noise and discretization are negligible asymptotically. Therefore, the iterates closely follow the trajectory of the ODE and ensure a.s. convergence to the globally asymptotically stable equilibrium. 
Using the two timescale approach \cite{borkar2008stochastic}, we consider 
(\ref{eq:primal_1}) for a fixed  $\boldsymbol{\tau}$. Let $M_1: \mathcal{R}^{|\mathcal{S}|} \to \mathcal{R}^{|\mathcal{S}|}$ be the following map 
\begin{equation}\label{eq:map_1}
\begin{split}
M_1(x)=&\sum \limits_{s'}P_{ss'}(\boldsymbol{\tau})[\sum_{i=0}^N \frac{\lambda_{i}(s)}{v(s)} R_i\mathbbm{1}_{\{a=A_2\}}-h(s)\\&+x(s')]-x(s^*), x \in \mathcal{R}^{|\mathcal{S}|}.
\end{split}
\end{equation}

Note that the knowledge of $P_{ss'}(\boldsymbol{\tau})$ and $ \lambda_{i}(s)$ are not required for the algorithm and is only required for analysis. For a fixed $\boldsymbol{\tau}$, 
 (\ref{eq:primal_1}) tracks the limiting ODE
$\dot{V}(t)=M_1(V(t))-V(t)$.    
$V(t)$ converges to the fixed point of $M_1(.)$ (determined using $M_1(V)=V$)\cite{konda1999actor} which is the asymptotically stable equilibrium of the ODE, as $t \to \infty$. Analogous methodologies are adopted in 
\cite{abounadi2001learning,konda1999actor}. Note that the approximation described in (\ref{eq:sigmoid}) does not impact the convergence
argument since the fixed point of $M_1(.)$ remains the same.
Next we establish that the value function and threshold vector iterates are bounded.
\begin{lemma}
The threshold vector and value function iterates 
are bounded a.s.
\end{lemma}
\begin{proof}
Let $M_0: \mathcal{R}^{|\mathcal{S}|} \to \mathcal{R}^{|\mathcal{S}|}$ be the following map 
\begin{equation}\label{eq:map_2}
    M_0(x)=\sum \limits_{s'}P_{ss'}(\boldsymbol{\tau})x(s')-x(s^*), x\in \mathcal{R}^{|\mathcal{S}|}.    
\end{equation}

Clearly, if the reward and cost functions are zero,  (\ref{eq:map_1}) is same as  (\ref{eq:map_2}). Also, $\lim_{b \to \infty}\frac{M_1(bV)}{b}=M_0(V)$. The globally asymptotically stable equilibrium of the ODE $\dot{V}(t)=M_0(V(t))-V(t)$ which is a scaled limit of ODE $\dot{V}(t)=M_1(V(t))-V(t)$,
is the origin. Boundedness of value functions and threshold vector iterates follow from \cite{borkar2000ode}
and (\ref{eq:dual_1}), respectively. 
\end{proof}
\begin{lemma}
$V_n-V^{\boldsymbol{\tau}_n}\to 0$, where $V^{\boldsymbol{\tau}_n}$ is the value function of states for $\boldsymbol{\tau}=\boldsymbol{\tau}_n$
a.s.
\end{lemma}
\begin{proof}
Since the threshold vector iterates are updated in a slower timescale, value function iterates in the faster timescale treat the threshold vector iterates 
as fixed. Therefore, iterations for $\boldsymbol{\tau}$ are
$\boldsymbol{\tau}_{n+1}=\boldsymbol{\tau}_n+\gamma(n)$,    
where $\gamma(n)=O(b(n))=o(a(n))$. Therefore, the limiting ODEs for value function and threshold vector iterates are $\dot{V}(t)=M_1(V(t))-V(t)$ and $\dot{\boldsymbol{\tau}}(t)=0$, respectively. It is sufficient to consider the ODE $\dot{V}(t)=M_1(V(t))-V(t)$ alone for a fixed $\boldsymbol{\tau}$ because $\dot{\boldsymbol{\tau}}(t)=0$. The rest of the proof is analogous to that of \cite{borkar2005actor}. 
\end{proof}
For the time being assume that $\sigma(\boldsymbol{\tau})$ is unimodal in $\boldsymbol{\tau}$. This is proved later.
The lemmas presented next 
establish that 
threshold vector iterates $\boldsymbol{\tau}_n$ converge to the optimal threshold vector $\boldsymbol{\tau}^*$ and hence, $(V_n,\boldsymbol{\tau}_n)$ converges to the optimal pair $(V,\boldsymbol{\tau}^* )$.
\begin{lemma}
The threshold vector iterates $\boldsymbol{\tau}_n\to \boldsymbol{\tau}^*$ a.s.
\end{lemma}
\begin{proof}
The limiting ODE
for (\ref{eq:dual_1}) is the gradient ascent scheme
    $\dot{\boldsymbol{\tau}}=\nabla \sigma(\boldsymbol{\tau})$.
Note that the gradient points inwards at ${\tau}(.)=0$ and ${\tau}(.)=W$. 
Since $\sigma(\boldsymbol{\tau})$ is unimodal in $\boldsymbol{\tau}$,
there does not exist any local maximum except $\boldsymbol{\tau}^*$ which is the global maximum. This concludes the proof of the lemma. 
\end{proof}
\begin{remark}
If the unimodality of the average reward with respect to the threshold vector does not hold, then convergence to only a local maximum can be guaranteed.
\end{remark}
\section{}\label{app:c}
\subsection{Proof of Lemma \ref{lemma:noninc}:}
\begin{proof}
Proof methodologies are similar to \cite{roy2019structure}.
\begin{equation}\label{eq:app_dec}
\begin{split}
    v_{n+1}(s)=&\sum_{i=1}^N \lambda_i \max\{R_i+v_n(s+1),v_n(s)\}+\\&\min\{s,m\}\mu v_n((s-1)^+)+(1-v(s))v_n(s)-h(s),
\end{split}
\end{equation}
and 
$\hat{v}_{i,n+1}(s)=\max \{v_n(s),R_i+v_n(s+1)\}.$
We know, $Dv_n(s)=v_n(s+1)-v_n(s)$.
We use induction to prove that $Dv_n(s)$ is decreasing in $n$. If $n=0$, $v_0(s)=0$ and $Dv_0(s)=0$.
Using (\ref{eq:app_dec}), it is easy to see that $Dv_1(s)< Dv_0(s)$ as $h(s+1)> h(s)$.
Assuming that the claim holds for any $n$, i.e., $Dv_{n+1}(s)< Dv_n(s)$, we need to prove that $Dv_{n+2}(s)< Dv_{n+1}(s)$. 
To analyze the second and third terms in (\ref{eq:app_dec}), we consider two cases.\\
(a)    $s\ge m$:
    $m\mu Dv_{n+1}((s-1)^+)+(1-\sum_{i=1}^N \lambda_i-m \mu)D v_{n+1}(s)$ is less than $m\mu D v_n((s-1)^+)+(1-\sum_{i=1}^N \lambda_i-m \mu)Dv_n(s)$.\\
(b)    $s<m$:
        $s \mu D v_{n+1}((s-1)^+)+(1-v(s+1)) D v_{n+1}(s)$ is less than  $s \mu D v_{n}((s-1)^+)+(1-v(s+1)) D v_{n}(s)$. \\
We proceed to prove that $D\hat{v}_{i,n+2}(s)\le D\hat{v}_{i,n+1}(s)$. Let at $(n+2)^{\rm {th}}$ iteration, maximizing actions for the admission of class-$i$ customers in states $s$ and $(s+1)$ be denoted by $a_{i,1}\in \{A_1,A_2\}$ and $a_{i,2}\in \{A_1,A_2\}$, respectively. Let $b_{i,1},b_{i,2}\in \{A_1,A_2\}$
be the maximizing actions in states $s$ and $(s+1)$, respectively, at $(n+1)^{\rm {th}}$ iteration.
It is not possible to have $a_{i,2}=A_2$ and $b_{i,1}=A_1$. If $b_{i,1}=A_1$, then $Dv_n(s)\le -R_i$. Therefore, we must have $Dv_n(s+1)< -R_i$ (Using Lemma \ref{lemma:a}). If $a_{i,2}=A_2$, then $Dv_{n+1}(s+1)\ge -R_i$ which contradicts the inductive assumption. Therefore, we consider the remaining cases. Note that if the inequality holds for any $a_{i,1}$ and $b_{i,2}$ for given $a_{i,2}$ and $b_{i,1}$, then the maximizing actions will satisfy the inequality too.\\
 (a) $a_{i,2}=b_{i,1}=A_1$: We choose $a_{i,1}=b_{i,2}=A_1$ to get 
     $D\hat{v}_{i,n+2}(s)- D\hat{v}_{i,n+1}(s)=D\hat{v}_{i,n+1}(s)- D\hat{v}_{i,n}(s)\le 0.$ \\
(b)     $a_{i,2}=b_{i,1}=A_2$: Proof is similar to the preceding case by choosing $a_{i,1}=b_{i,2}=A_2$.\\
(c) $a_{i,2}=A_1, b_{i,1}=A_2$: Choose 
     $a_{i,1}=A_2$ and $b_{i,2}=A_1$.
     $D\hat{v}_{i,n+2}(s)- D\hat{v}_{i,n+1}(s)
     = v_{n+1}(s+1)-R_i-v_{n+1}(s+1)-v_n(s+1)+R_i+v_n(s+1)=0.$\\
Thus, $D\hat{v}_{i,n+2}(s)\le D\hat{v}_{i,n+1}(s)$. Since this holds for every $i$ and $h(s)$ is independent of $n$, this concludes the proof. 
\end{proof}
\subsection{Proof of Lemma \ref{lemma:unimodal}:}
\begin{proof}
Proof idea is similar to that of \cite{roy2019low}.
We prove this lemma for $i^{\rm {th}}$ component of the threshold vector (viz., ${\tau}(i)$). 
If the optimal action for the admission of class-$i$
customers in state $s$ is $A_1$, then $V(s+1)-V(s)\le -R_i$. Since VIA converges to the optimal threshold vector $\boldsymbol{\tau}^*$ in finite time, $\exists N_0>0$ such that $\forall n \ge N_0, v_n(s+1)-v_n(s)\le -R_i, \forall s \ge {\tau}^*(i)$ and 
$v_n(s+1)-v_n(s)\ge -R_i, \forall s < {\tau}^*(i)$. Let $U_{i,n},n\ge 1$ be the optimal threshold for class-$i$
customers at $n^{\rm{th}}$ iteration of VIA.
Hence, $U_{i,n}=\min\{s\in \mathbb{N}_0: v_n(s+1)-v_n(s)\le -R_i\}$. If no values of $s$ satisfies the inequality, then $U_{i,n}=m+B$. Since $v_n(s+1)-v_n(s)$ is decreasing in $n$ (Lemma \ref{lemma:noninc}), $U_{i,n}$ is monotonically decreasing in $n$, and $\lim \limits_{n \to \infty} U_{i,n} ={\tau}^*(i)$.\par 
Consider a re-designed problem where for a given threshold vector $\boldsymbol{\tau}'$ such that ${\tau}^*(i)\le {\tau}'(i)\le m+B$, action $A_1$ is not allowed in any state $s<{\tau}'(i)$. Note that Lemma $\ref{lemma:noninc}$ holds for this re-designed problem also. Let $n_{{\tau}'(i)}$  
be the first iteration of VIA when the threshold reduces to ${\tau}'(i)$. The value function iterates for the original and re-designed problem are same for $n\le n_{{\tau}'(i)}$ because in the original problem also $A_1$ is never chosen as the optimal action in states $s<{\tau}'(i)$ at these iterations. Hence, $n_{{\tau}'(i)}$ must be finite and the inequality     $v_n({\tau}'(i)+1)-v_n({\tau}'(i))\le -R_i$ is true for both the problems after $n_{{\tau}'(i)}$ iterations.
Using Lemma \ref{lemma:noninc}, this inequality holds $\forall n \ge {\tau}'(i)$. Therefore, in the re-designed problem, $U_{i,n}$ converges to ${\tau}'(i)$. Thus, the threshold policy with ${\tau}'(i)$ is superior than that with ${\tau}'(i)+1$. Since this holds for arbitrary choice of ${\tau}'(i)$,
average reward monotonically decreases with ${\tau}'(i)$, $\forall {\tau}'(i)>{\tau}^*(i)$.\par 
If we have $\sigma(\boldsymbol{\tau})\ge \sigma(\boldsymbol{\tau}+e_i)$ (where $e_i \in \mathbb{R}^N$ is a vector with all zeros except the $i^{\rm{th}}$ element being `1'), then we must have 
${\tau}(i)\ge {\tau}^*(i)$. Therefore, $\sigma(\boldsymbol{\tau}+e_i)\ge \sigma(\boldsymbol{\tau}+2e_i)$. Hence, the average reward is unimodal in ${\tau}(i)$.
Since the proof holds for any $i$, this concludes the proof of the lemma. 
\end{proof}
\section*{Acknowledgment}
Works of Arghyadip Roy, Abhay Karandikar and Prasanna Chaporkar are supported   by   the    Ministry of Electronics and Information Technology (MeitY), Government of India  as part of ``5G Research and Building Next Gen Solutions for Indian Market'' project. 
Work of Vivek Borkar is supported by the CEFIPRA grant for ``Machine Learning for Network Analytics'' and a J. C. Bose Fellowship.
\bibliography{struc} 

\end{document}